\documentclass{article} 

\usepackage{iclr2025_conference,times}

\definecolor{darkred}{RGB}{150, 0, 0}
\definecolor{darkgreen}{RGB}{0, 150, 0}
\definecolor{darkblue}{RGB}{0, 0, 150}
\usepackage{hyperref}
\usepackage{url}

\usepackage{tikz}
\usetikzlibrary{patterns}

\usepackage{amssymb,amsmath,amsthm}
\usepackage{mathtools} 
\usepackage{enumerate}
\usepackage{dsfont}
\usepackage{nicefrac}
\usepackage{xspace}
\usepackage{thmtools}
\usepackage{thm-restate}

\usepackage[english=american]{csquotes} 

\definecolor{lightgreen}{rgb}{0.8, 1.0, 0.8}

\usepackage[colorinlistoftodos]{todonotes}

\newtheorem{thm}{Theorem} 
\newtheorem{prop}[thm]{Proposition} 
\newtheorem{cor}[thm]{Corollary} 
\newtheorem{lem}[thm]{Lemma} 
\newtheorem{rem}[thm]{Remark} 
\newtheorem{conjecture}[thm]{Conjecture}
\newtheorem{question}[thm]{Question}

\theoremstyle{definition}
\newtheorem{example}[thm]{Example}

\newcommand{\N}{\mathbb{N}}
\newcommand{\R}{\mathbb{R}}
\newcommand{\Q}{\mathbb{Q}}
\newcommand{\Z}{\mathbb{Z}}

\newcommand{\ReLU}{\operatorname{ReLU}}
\newcommand{\maxpool}[1][n]{\ensuremath{\max\{0, x_1, \ldots, x_{#1}\}}}
\newcommand{\maxpoolshort}[1][n]{\ensuremath{F_{#1}}}

\newcommand{\SU}{\operatorname{SU}}

\newcommand{\Vol}{\operatorname{Vol}}
\newcommand{\vol}{\operatorname{vol}}
\newcommand{\V}{\operatorname{V}}

\newcommand{\conv}{\operatorname{conv}}
\DeclarePairedDelimiter\ceil{\lceil}{\rceil}

\newcommand{\cP}{\mathcal{P}}
\newcommand{\cPd}[1][d]{\mathcal{P}_{#1}} 
\newcommand{\cX}{\mathcal{X}}

\newcommand{\bigO}{\mathcal{O}}
\newcommand{\define}{\coloneqq}
\newcommand{\T}{^{\top}}
\newcommand{\primepowsentence}{Let $d = p^t \leq n$ be a power of a prime number $p$, with $t \in \N$.\xspace}

\newcommand{\myparagraph}[1]{\textbf{#1}\quad}

\title{On the Expressiveness of Rational ReLU\\ Neural Networks With Bounded Depth}

\author{%
  Gennadiy Averkov\\
  BTU Cottbus-Senftenberg\\
  \texttt{averkov@b-tu.de}\\
  \And
  Christopher Hojny\\
  TU Eindhoven\\
  \texttt{c.hojny@tue.nl}\\
  \And
  Maximilian Merkert\\
  TU Braunschweig \\
  \texttt{m.merkert@tu-braunschweig.de}
}

\iclrfinalcopy 
\begin{document}

\maketitle

\begin{abstract}
  To confirm that the expressive power of ReLU neural networks grows with their depth, the function $\maxpoolshort = \max \{0,x_1,\ldots,x_n\}$ has been considered in the literature.
  A conjecture by Hertrich, Basu, Di Summa, and Skutella [NeurIPS~2021] states that any ReLU network that exactly represents~$\maxpoolshort$ has at least $\ceil{\log_2 (n+1)}$ hidden layers.
  The conjecture has recently been confirmed for networks with integer weights by Haase, Hertrich, and Loho [ICLR~2023].

  We follow up on this line of research and show that, within ReLU networks whose weights are decimal fractions, $\maxpoolshort$ can only be represented by networks with at least~$\ceil{\log_3 (n+1)}$ hidden layers.
  Moreover, if all weights are~$N$-ary fractions, then $\maxpoolshort$ can only be represented by networks with at least $\Omega( \frac{\ln n}{\ln \ln N})$ layers.
  These results are a  partial confirmation of the above conjecture for rational ReLU networks, and provide the first non-constant lower bound on the depth of practically relevant ReLU networks.
\end{abstract}

\section{Introduction} 

An important aspect of designing neural network architectures is to
understand which functions can be exactly represented by a
specific architecture.
Here, we say that a neural network, transforming~$n$
input values into a single output value, \emph{(exactly) represents} a
function~$f\colon \R^n \to \R$ if, for every input~$x \in \R^n$, the neural
network reports output~$f(x)$.
Understanding the expressiveness of neural network architectures can help to,
among others, derive algorithms~\citep{abmm,KhalifeChengBasu2024,HertrichSering2024} and
complexity
results~\citep{GoelEtAl2021,FroeseEtAl2022,BertschingerEtAl2023,FroeseHertrich2023}
for training networks.

One of the most popular classes of neural networks are feedforward neural
networks with ReLU activation \citep{GoodfellowEtAl2016}.
Their capabilities to \emph{approximate} functions is well-studied and led
to several so-called universal approximation theorems, e.g.,
see~\citep{Cybenko1989,Hornik1991}.
For example, from a result by \citet{Leshno1993} it follows that any continuous function
can be approximated arbitrarily well by ReLU networks with a single hidden
layer.
In contrast to approximating functions, the understanding of which
functions can be \emph{exactly} represented by a neural network is much
less mature.
A central result by \citet{abmm} states that the class of functions that
are exactly representable by ReLU networks is the class of continuous
piecewise linear (CPWL) functions.
In particular, they show that every CPWL function with~$n$ inputs can be
represented by a ReLU network with~$\ceil{\log_2(n+1)}$ hidden layers.
It is an open question though for which functions this number of hidden
layers is also necessary.

An active research field is therefore to derive lower bounds on the
number of required hidden layers.
\citet{abmm} show that two hidden layers are necessary and sufficient to
represent~$\max\{0,x_1,x_2\}$ by a ReLU network.
However, there is no single function which is known to require more than two hidden layers in an exact representation.
In fact, \citet{hbds} formulate the following conjecture.
\begin{conjecture}
  \label{conj:central}
  For every integer~$k$ with~$1 \leq k \leq \ceil{\log_2(n+1)}$, there exists a
  function~$f \colon \R^n \to \R$ that can be represented by a ReLU network
  with~$k$ hidden layers, but not with~$k-1$ hidden layers.
\end{conjecture}
\citet{hbds} also show that this conjecture is equivalent to the statement
that any ReLU network representing~$\max\{0,x_1,\dots,x_{2^k}\}$
requires~$k+1$ hidden layers.
That is, if the conjecture holds true, the lower bound of~$\ceil{\log_2(n+1)}$ by
\citet{abmm} is tight.
While Conjecture~\ref{conj:central} is open in general, it has been
confirmed for two subclasses of ReLU networks, namely networks all of whose weights only take
integer values~\citep{hhl} and, for~$n=4$, so-called $H$-conforming
neural networks~\citep{hbds}.

In this article, we follow this line of research by deriving a non-constant
lower bound on the number of hidden layers in ReLU networks all of whose
weights are~$N$-ary fractions.
Recall that a rational number is an~$N$-ary fraction if it can be written as~$\frac{z}{N^t}$ for
some integer~$z$ and non-negative integer~$t$.

\begin{thm}
  \label{thm:main1}
  Let~$n$ and~$N$ be positive integers, and let~$p$ be a prime number
  that does not divide~$N$.
  Every ReLU network with weights being~$N$-ary fractions requires at
  least~$\ceil{\log_p (n + 1)}$ hidden layers to exactly represent the
  function~$\max\{0,x_1,\dots,x_n\}$.
\end{thm}

\begin{cor}
  \label{cor:log3}
  Every ReLU network all of whose weights are decimal fractions requires at
  least~$\ceil{\log_3(n+1)}$ hidden layers to exactly represent~$\maxpool$.
\end{cor} 

While Theorem~\ref{thm:main1} does not resolve
Conjecture~\ref{conj:central} because it makes no statement about general real weights,
note that in most applications floating point arithmetic is
used \citep{IEEEStd2019}.
That is, in neural network architectures used in practice, one is actually
restricted to weights being~$N$-ary fractions.
Moreover, when quantization, see, e.g., \citep{GhomlamiEtAl2022} is used to make neural networks more
efficient in terms of memory and speed, weights can become low-precision
decimal numbers, cf., e.g., \citep{NagelEtAl2020}.
Consequently, Theorem~\ref{thm:main1} provides, to the best of our
knowledge, the first non-constant lower bound on the depth of practically
relevant ReLU networks.

Relying on Theorem~\ref{thm:main1}, we also derive the following lower bound.

\begin{thm}
  \label{thm:main2}
  There is a constant~$C>0$ such that, for all integers~$n, N \geq 3$, every ReLU network with weights being~$N$-ary fractions that represents~$\maxpool$ has depth at least~$C \cdot \frac{\ln n}{\ln\ln N}$.
\end{thm}

Theorem~\ref{thm:main2}, in particular, shows that there is no constant-depth ReLU network that exactly represents~$\maxpool$ if all weights are rational numbers all having a common denominator~$N$.

In view of the integral networks considered by \citet{hhl}, we stress that
our results do not simply follow by scaling integer weights to rationals,
which has already been discussed in~\citet[Sec.~1.3]{hhl}.
We therefore extend the techniques by \citet{hhl} to make use of number
theory and polyhedral combinatorics to prove our results that cover
standard number representations of rationals on a computer.

\myparagraph{Outline}
To prove our main results, Theorems~\ref{thm:main1} and~\ref{thm:main2},
the rest of the paper is structured as follows.
First, we provide some basic definitions regarding neural networks that we use throughout the article,
and we provide a brief overview of related literature.
Section~\ref{sec:strategy} then provides a short summary of our overall strategy to prove
Theorems~\ref{thm:main1} and~\ref{thm:main2} as well as some basic notation.
The different concepts of polyhedral theory and volumes needed in our proof strategy are detailed
in Section~\ref{sec:polytheory}, whereas Section~\ref{sec:NNcharac} recalls a characterization
of functions representable by a ReLU neural network from the literature, which forms the basis of our proofs.
In Section~\ref{sec:results}, we derive various properties of polytopes associated with functions
representable by a ReLU neural network, which ultimately allows us to prove our main results
in Section~\ref{sec:resultsFractions}.
The paper is concluded in Section~\ref{sec:conclusions}.

\myparagraph{Basic Notation for ReLU Networks}
To describe the neural networks considered in this article, we introduce
some notation.
We denote by~$\Z$, $\N$, and~$\R$ the sets of integer, positive integer,
and real numbers, respectively. Moreover,~$\Z_+$ and~$\R_+$ denote the sets of non-negative integers and reals, respectively.

Let~$k \in \Z_+$.
 A~\emph{feedforward neural network with rectified linear units (ReLU)} (or simply \emph{ReLU network} in the following) with~$k+1$ layers can be described by $k+1$ affine transformations $t^{(1)}\colon \R^{n_{0}} \to \R^{n_{1}},\ldots,t^{(k+1)}\colon \R^{n_{k}} \to \R^{n_{k+1}}$.
It \emph{exactly represents} a
function~$f\colon \R^n \to \R$ if and only if~$n_0 = n$, $n_{k+1} = 1$,
and the alternating composition
\[
  t^{(k+1)} \circ \sigma \circ t^{(k)} \circ \sigma \circ
  \dots
  \circ t^{(2)} \circ \sigma \circ t^{(1)}
\]
coincides with~$f$, where, by slightly overloading notation, $\sigma$ denotes the component-wise application of the \emph{ReLU activation function} $\sigma\colon \R \to \R$, $\sigma(x) = \max\{0,x\}$ to vectors in any dimension.
For each~$i \in \{1,\dots,k+1\}$ and $x \in \R^{n_{i-1}}$, let~$t^{(i)}(x) = A^{(i)}x + b^{(i)}$ for
some~$A^{(i)} \in \R^{n_i \times n_{i-1} }$ and~$b^{(i)} \in \R^{n_i}$.
The entries of~$A^{(i)}$ are called \emph{weights} and those of~$b^{(i)}$ are
called \emph{biases} of the network.
The network's \emph{depth} is~$k+1$, and the \emph{number of hidden layers} is~$k$.

The set of all functions~$\R^n \to \R$ that can be represented exactly by a
ReLU network of depth~$k+1$ is denoted by~$\ReLU_n(k)$.
Moreover, if~$R \subseteq \R$ is a ring, we denote by~$\ReLU_n^R(k)$ the
set of all functions~$\R^n \to \R$ that can be represented exactly by a
ReLU network of depth~$k+1$ all of whose weights are contained in~$R$.
Throughout this paper, we will mainly work with the rings~$\Z$, $\R$, or
the ring of~$N$-ary fractions.

The set $\ReLU_n^R(0)$ is the set of affine functions $f(x_1,\ldots,x_n) = b + a_1 x_1 + \cdots + a_n x_n$ with $b \in \R$, and $a_1,\ldots,a_n \in R$. It can be directly seen from the definition of ReLU networks that, for $ k \in \N$, one has $f \in \ReLU_n^R(k)$  if and only if $f(x) = u_0 + u_1 \max \{0,g_1(x) \} + \cdots + u_m \max \{0,g_m(x)\}$ holds for some $m \in \N$, $u_0 \in \R, u_1,\ldots,u_m \in R$, and functions $g_1,\ldots,g_m \in \ReLU_n^R(k-1)$. 

\myparagraph{Related Literature}
Regarding the expressiveness of ReLU networks, \citet{hbds} show that four
layers are needed to exactly
represent~$\maxpool[4]$ if the network satisfies the technical condition of
being~$H$-conforming.
By restricting the weights of a ReLU network to be integer, \citet{hhl}
prove that~$\ReLU_n^\Z(k-1) \subsetneq \ReLU_n^\Z(k)$ for
every~$k \leq \ceil{\log_2(n+1)}$.
In particular, $\maxpool[2^k] \notin \ReLU_{2^k}^\Z(k)$.
If the activation function is changed from ReLU to~$x \mapsto
\mathds{1}_{\{x>0\}}$, \citet{KhalifeChengBasu2024} show that two hidden layers are both necessary
and sufficient for all functions representable by such a network.

If one is only interested in approximating a function,
\citet{SafranReichmanValiant2024} show that~$\maxpool$ can be approximated
arbitrarily well by~$\ReLU_n^\Z(2)$-networks of width $n(n + 1)$ with respect to the $L_2$  norm for continuous distributions.
By increasing the depth of these networks, they also derive upper bounds on
the required width in such an approximation.
The results by~\citet{SafranReichmanValiant2024} belong to the class of
so-called universal approximation theorems, which describe the ability to
approximate classes of functions by specific types of neural networks, see,
e.g., \citep{Cybenko1989,Hornik1991,Barron1993,Pinkus1999,KidgerLyons2020}.
However, \citet{VardiShamir2020} show that there are significant theoretical barriers for depth-separation results for polynomially-sized $\ReLU_n(k)$-networks for $k\geq3$, by establishing links to the separation of threshold circuits as well as to so-called natural-proof barriers.
When taking specific data into account, \citet{LeeMammadovYe2024} derive
lower and upper bounds on both the depth and width of a neural network that
correctly classifies a given data set.
More general investigations of the relation between the width and depth of a neural network are discussed,
among others, by \citet{abmm,EldanShamir2016,Hanin2019,pmlr-v70-raghu17a,SafranShamir2017,pmlr-v49-telgarsky16}.

\section{Proof Strategy and Theoretical Concepts}
\label{sec:strategy}

To prove Theorems~\ref{thm:main1} and~\ref{thm:main2}, we extend the ideas
of~\citet{hhl}.
We therefore provide a very concise summary of the arguments of~\citet{hhl}.
Afterwards, we briefly mention the main ingredients needed in our proofs,
which are detailed in the following subsections.

A central ingredient for the results by~\citet{hhl} is a polyhedral
characterization of all functions in~$\ReLU_n(k)$, which has been derived by~\citet{hertrich:phd}.
This characterization links functions representable by a ReLU network and
so-called support functions of polytopes~$P \subseteq \R^n$ all of whose
vertices belong to~$\Z^n$, so-called \emph{lattice polytopes}.
It turns out that the function~$\maxpool$ in
Theorems~\ref{thm:main1} and~\ref{thm:main2} can be expressed as the
support function of a particular lattice polytope~$P_n \subseteq \R^n$.
By using a suitably scaled version~$\Vol_n$ of the classical Euclidean
volume in~$\R^n$, one can achieve~$\Vol_n(P) \in \Z$ for all lattice polytopes~$P \subseteq \R^n$.
\citet{hhl} then show that, if the support function~$h_P$ of a lattice polytope~$P \subseteq \R^n$ can be
exactly represented by a ReLU network with~$k$ hidden layers, all faces of~$P$ of dimension
at least~$2^{k}$ have an even normalized volume.
For~$n = 2^k$, however, $\Vol_n(P_n)$ is odd.
Hence, its support function cannot be represented by a ReLU network with~$k$ hidden layers.

We show that the arguments of \citet{hhl} can be adapted by replacing the divisor~$2$
with an arbitrary prime number $p$. Another crucial insight is that the
theory of mixed volumes can be used to analyze the behavior of 
$\Vol_n(A+B)$ for the Minkowski sum
$A + B \define \{ a + b \colon a \in A, b \in B\}$ of lattice polytopes
$A, B \subset \R^n$.
The Minkowski-sum operation is also involved in the polyhedral characterization of \citet{hertrich:phd}, and so it is also
used by \citet{hhl}, who provide a version of Theorem~\ref{thm:main1} for integer
weights.
They, however, do not directly use mixed volumes.  A key observation used in
our proofs, and obtained by a direct application of mixed volumes, is that
the map associating to a lattice polytope $P$ the coset of $\Vol_n(P)$
modulo a prime number $p$ is additive when $n$ is a power of $p$.
Combining these ingredients yields Theorems~\ref{thm:main1} and~\ref{thm:main2}.

\myparagraph{Some Basic Notation}
The standard basis vectors in $\R^n$ are denoted by $e_1,\ldots,e_n$,
whereas~$0$ denotes the null vector in~$\R^n$.
Throughout the article, all vectors~$x \in \R^n$ are column vectors, and we
denote the transposed vector by~$x\T$.
If~$x$ is contained in the integer lattice~$\Z^n$, we call it a \emph{lattice point}.
For vectors~$x, y \in \R^n$, their scalar product is given by~$x\T y$.
For $m \in \N$, we will write $[m]$ for the set~$\{1,\ldots,m\}$.
The convex-hull operator is denoted by~$\conv$, and the base-$b$ logarithm
by~$\log_b$, while the natural logarithm is denoted~$\ln$.

The central function of this article is~$\maxpool$, which we abbreviate by~$\maxpoolshort$.

\subsection{Basic Properties of Polytopes and Lattice Polytopes}
\label{sec:polytheory}

As outlined above, the main tools needed to prove Theorems~\ref{thm:main1} and~\ref{thm:main2}
are polyhedral theory and different concepts of volumes.
This section summarizes the main concepts and properties that we need in our argumentation in Section~\ref{sec:results}. For more background, we refer the reader to the monographs \citep{br20,HugWeil2020,schneider}.

\myparagraph{Polyhedra, Lattice Polyhedra, and Their Normalized Volume}
A \emph{polytope} $P \subseteq \R^n$ is the convex hull
$\conv(p_1,\ldots,p_m)$ of finitely many points $p_1,\ldots,p_m \in \R^n$.
We introduce the family
\[
  \cP(S) \define \{ \conv(p_1,\ldots,p_m) \colon m \in \N, \ p_1,\ldots,p_m \in S \}
\]
of all non-empty polytopes with vertices in $S \subseteq \R^n$.
Thus, $\cP(\R^n)$ is the family of all polytopes in $\R^n$ and $\cP(\Z^n)$ is the family of all \emph{lattice polytopes} in~$\R^n$.
For $d \in \{0,\ldots,n\}$, we also introduce the family
\[
  \cPd(S) \define \{ P \in \cP(S) \colon \dim(P) \leq d \}.
\]
of polytopes of dimension at most $d$, where the dimension of a polytope $P$ is defined as the dimension of its affine hull, i.e., the smallest affine subspace of $\R^n$ containing $P$.
The \emph{Euclidean volume} $\vol_n$ on $\R^n$ is the $n$-dimensional Lebesgue measure, scaled so that~$\vol_n$ is equal to $1$ on the unit cube $[0,1]^d$. Note that measure-theoretic subtleties play no role in our context since we restrict the use of $\vol_n$  to $\cP(\R^n)$.  The \emph{normalized volume} $\Vol_n$ in $\R^n$ is the $n$-dimensional Lebesgue measure normalized so that $\Vol_n$ is equal to $1$ on the \emph{standard simplex} $\Delta_n \define \conv(0,e_1,\ldots,e_n)$. Clearly, $\Vol_n = n! \cdot \vol_n$ and $\Vol_n$ takes non-negative integer values on lattice polytopes.

\myparagraph{Support Functions}
For a polytope~$P = \conv(p_1,\ldots,p_m) \subseteq \R^n$, its \emph{support function} is
\[
  h_P(x) \define \max \{x\T y : y \in P\},
\]
and it is well-known that~$h_P(x) = \max\{p_1\T x,\ldots,p_m\T x\}$.
Consequently, $\maxpool$ from Theorems~\ref{thm:main1} and~\ref{thm:main2} is the support function of~$\Delta_n$.

\myparagraph{Mixed Volumes}
For sets $A,B \subseteq \R^n$, we introduce the \emph{Minkowski sum} 
\[
	A+ B\coloneqq\{ a + b \colon a \in A , b \in B\}
\] 
and the multiplication 
\[
	\lambda A = \{ \lambda a \colon a \in A\}
\] of $A$ by a non-negative factor $\lambda \in \R_+$. For an illustration of the Minkowski sum, we refer to Figure~\ref{fig:example:Msum}.
Note that, if~$S \in \{\R^n, \Z^n\}$ and~$A,B \in \cP(S)$, then~${A + B \in \cP(S)}$, too.
If $A$ and $B$ are (lattice) polytopes, then $A+B$ is also a (lattice) polytope, and the support functions of $A, B$ and $A+B$ are related by $h_{A+B} = h_A + h_B$.

If $(G,+)$ is an Abelian semi-group (i.e., a set with an associative and commutative binary operation), we call a map $\phi\colon \cP(\R^n) \to G$ \emph{Minkowski additive} if the Minkowski addition on $\cP(\R^n)$ gets preserved by $\phi$ in the sense that $\phi(A+B) = \phi(A) + \phi(B)$ holds for all $A,B \in \cP(\R^n)$.

The following is a classical result of Minkowski. 

\begin{thm}[see, e.g., {\citep[Ch.~5]{schneider}}]  \label{thm:mixed:vol}
	There exists a unique functional, called the \emph{mixed volume},
	\[
		\V \colon \cP(\R^n)^n \to \R, 
	\]
	with the following properties valid for all $P_1,\ldots,P_n,A,B \in \cP(\R^n)$ and $\alpha, \beta \in \R_+$: 
	\begin{enumerate}[(a)]
		\item $\V$ is invariant under permutations, i.e. $\V(P_1,\ldots,P_n) = \V(P_{\sigma(1)},\ldots,P_{\sigma(n)} )$ for every permutation $\sigma$ on $[n]$. 
		\item $\V$ is Minkowski linear in all input parameters,
                  i.e., for all~$i \in [n]$, it holds that
		\begin{align*} 
			\V(P_1,\ldots P_{i-1},\alpha A + \beta B, P_{i+1}, \ldots, P_n)  =  & \alpha \V(P_1,\ldots P_{i-1}, A , P_{i+1}, \ldots, P_n)  \\   + & \beta \V(P_1,\ldots P_{i-1},B, P_{i+1}, \ldots, P_n)
		\end{align*} 
		\item $\V$ is equal to $\Vol_n$ on the diagonal, i.e., $\V(A,\ldots,A)  = \Vol_n(A)$. 
	\end{enumerate} 
\end{thm} 

We refer to Chapter~5 of the monograph by \citet{schneider} on the Brunn-Minkowski theory for more information on mixed volumes, where also an explicit formula for the mixed volume is presented. Our definition of the mixed volume differs by a factor of $n!$ from the definition in \citet{schneider} since we use the normalized volume $\Vol_n$ rather than the Euclidean volume $\vol_n$ to fix $\V(P_1,\ldots,P_n)$ in the case $P_1 = \ldots = P_n$. Our way of introducing mixed volumes is customary in the context of algebraic geometry. It is known that, for this normalization, $\V(P_1,\ldots,P_n) \in \Z_+$ when $P_1,\ldots,P_n$ are lattice polytopes; see, for example, \citep[Ch.~4.6]{MacLaganSturmfels}.
From the defining properties one can immediately see that, for $A,B \in
\cP(\R^n)$, one has the analogue of the binomial formula, which we will
prove in Appendix~\ref{sec:ProofBinomial} for the sake of completeness:
\begin{equation} \label{MV:binomial}
	\Vol_n(A+B) =  \sum_{i=0}^n \binom{n}{i} \V(\underbrace{A,\ldots,A}_i,\underbrace{B,\ldots,B}_{n-i}).
\end{equation}

\myparagraph{Normalized Volume of Non-Full-Dimensional Polytopes}
So far, we have introduced the normalized volume~$\Vol_n\colon \cP(\R^n) \to \R_+$, i.e., if~$P \in \cP(\R^n)$ is not full-dimensional, then~$\Vol_n(P) = 0$. 
We also associate with a polytope $P \in \cP_d(\Z^n)$ of dimension at most $d$ an appropriately normalized $d$-dimensional volume by extending the use of $\Vol_d\colon \cP(\Z^d) \to \Z_+$ to $\Vol_d\colon \cP_d( \Z^n) \to \Z_+$.  
In the case $\dim(P)<d$, we define $\Vol_d(P) = 0$. If $d=0$, let $\Vol_d(P)=1$. In the non-degenerate case $d = \dim(P) \in \N$, we fix $Y$ to be the affine hull of $P$ and consider a bijective affine map $T\colon \R^d \to Y$ satisfying $T(\Z^d) = Y \cap \Z^n$. For such choice of $T$, we have $T^{-1}(P) \in \cP(\Z^d)$. It turns out that the $d$-dimensional volume of $T^{-1}(P)$ depends only on $P$ and not on $T$ so that we define $\Vol_d(P) \define \Vol_d(T^{-1}(P))$.   Based on \cite[Corollary~3.17 and \S5.4]{br20}, there is the following intrinsic way of introducing $\Vol_d(P)$. Let $G(P)$ denote the number of lattice points in $P$. Then, for $t \in \Z_+$, one has $\Vol_d(P) \define d! \cdot  \lim_{t \to \infty} \frac{1}{t^d}G( t P)$. 

\begin{rem} \label{rem:d-dim:vol:theory}
 For every $d$-dimensional affine subspace $Y \subseteq \R^n$ which is affinely spanned by $d+1$ lattice points, we can define $\Vol_d$ for every polytope $P \in \cP(Y)$, which is not necessarily a lattice  polytope, by the same formula $\Vol_d(P)  \define \Vol_d(T^{-1}(P))$, using an auxiliary map $T\colon \R^d \to Y$ described above. Consequently,  by replacing the dimension $n$ with $d$ and the family of polytopes $\cP(\R^n)$ with the family $\cP(Y)$ in Minkowski's Theorem~\ref{thm:mixed:vol}, we can introduce the notion of mixed volumes for polytopes in $\cP(Y)$. More specifically, we will make use of the mixed volumes of lattice polytopes in $\cP(Y \cap \Z^n)$.  
\end{rem}

\myparagraph{Normalized Volume of the Affine Join}
The following proposition, borrowed from \citet{hhl}, addresses the divisibility properties of the convex hull of the union of lattice polytopes that lie in skew affine subspaces. 

\begin{prop}[{\citealt[Lemma~6]{hhl}}] \label{prop:vol:join}
	Let $A, B \in \cP(\Z^n)$ be polytopes of dimensions $i \in \Z_+$ and $j \in \Z_+$, respectively, such that $P \define \conv(A \cup B)$ is of dimension $i+j+1$. Then $\Vol_{i+j}(P)$ is divisible by $\Vol_i(A) \Vol_j(B)$. In particular, if $i=0$, then $P$ is a pyramid over $B$ whose normalized volume $\Vol_{1+j}(B)$ is divisible by the normalized volume $\Vol_j(B)$ of its base $B$. 
\end{prop}

For an example illustration, see Figure~\ref{fig:example:convunion}. Since $P_1$ and $P_2$ lie in skew affine subspaces, Proposition~\ref{prop:vol:join} applies. Indeed, $\Vol_3(\conv(P_1 \cup P_2)) = 12$ is divisible by $\Vol_2(P_1)=6$ (and $\Vol_0(P_2)=1$).

\subsection{A Polyhedral Criterion for Functions Representable With $k$ Hidden Layers} 
\label{sec:NNcharac}

Next to the geometric concepts that we discussed before, the second main building block of our proofs is the polyhedral characterization of~$\ReLU_n(k)$ by~\citet{hertrich:phd}.
In the following, we introduce the necessary concepts and present Hertrich's characterization.

Note that~$\maxpoolshort$ is \emph{positively homogeneous}, i.e., for all~$\lambda \in \R_+$ and~$x \in \R^n$, one has~$\maxpoolshort(\lambda x) = \lambda \maxpoolshort(x)$.
For positively homogeneous functions~$f$, \citet{hbds} show that~$f \in \ReLU_n(k)$ if and only if there exists a ReLU network of depth~$k+1$ all of whose biases are~0.
This result easily generalizes to ReLU networks with weights being restricted to a ring~$R$.
We therefore denote by~$\ReLU_n^{R,0}(k)$ the set of all $n$-variate positively homogeneous functions representable by ReLU networks with $k$ hidden layers, weights in $R$, and all biases being~0.
Moreover, $\ReLU_n^{R,0} \define \bigcup_{k = 0}^\infty \ReLU_n^{R,0}(k)$.

To formulate the characterization by~\citet{hertrich:phd}, we define the \emph{sum-union closure} for a family of polytopes $\cX$ in $\R^n$ as
\[
  \SU( \cX )\coloneqq \left\{ \sum_{i=1}^m \conv(A_i \cup B_i) \colon m \in \N, \ A_i, B_i \in \cX, i \in [m] \right\}.
\] 
The $k$-fold application of the operation gives the \emph{$k$-fold sum-union closure} $\SU^k (\cX)$.  In other words, $\SU^0(\cX) = \cX$ and $\SU^k (\cX) = \SU( \SU^{k-1}(\cX))$ for $k \in \N$.  We will apply the $k$-fold sum-union closure to $\cP_0(S)$, the set of all $0$-dimensional polytopes of the form $\{s\}$, with $s \in S$. 

The set $\SU^k(\cX)$ forms a semi-group with respect to Minkowski-addition since, directly from the representation of elements of $\SU^k(\cX)$ as sums with arbitrarily many summands, one sees that $\SU^k(\cX)$ is closed under Minkowski addition.

\begin{figure}
\centering
\begin{minipage}[t]{0.48\textwidth}
\centering
\begin{tikzpicture}[scale=0.7]
	\def\xMin{-1}%
    \def\xMax{5}%
	\def\yMin{-1}%
    \def\yMax{3}%
    
    \draw [color=gray!50] [step=1] (\xMin+0.2,\yMin+0.2) grid (\xMax-0.2,\yMax-0.2);    
			
	\draw[->] (0,0,0) -- (\xMax-0.5,0,0) node[anchor=west]{$x$};
	\draw[->] (0,0,0) -- (0,\yMax-0.5,0) node[anchor=west]{$y$};
	\draw[->] (0,0,0) -- (0,0,3.5) node[anchor=west]{$z$};
	\draw (0,-0.1,1) -- (0,0.1,1);

	\draw[very thick, fill=gray!50] (0,0,0) -- (3,0,0) -- (3,1,0) -- (0,1,0) -- cycle;
	\node[very thick, draw=black, fill=gray!50, circle, inner sep=0pt,minimum size=5pt] at (0,0,2){};
	\draw[very thick, fill=gray!50, opacity=0.8] (0,0,2) -- (0,1,0) -- (3,1,0) -- cycle;
	\draw[very thick, fill=gray!50, opacity=0.8] (0,0,2) -- (3,1,0) -- (3,0,0) -- cycle;
	\node at (1.2,0.5) {$P_1$};
	\node[label=above left:$P_2$, draw=black, fill=gray!50, circle, inner sep=0pt,minimum size=5pt] at (0,0,2){};
	\draw[very thick] (0,0,2) -- (0,1,0);
	\draw[very thick] (0,0,2) -- (3,1,0);
	\draw[very thick] (0,0,2) -- (3,0,0);
	\node at (3,0,2) {$\conv(P_1 \cup P_2)$};
\end{tikzpicture}

\caption{Illustration of the convex hull of a polytope and a point, relating to Proposition~\ref{prop:vol:join}.
}
\label{fig:example:convunion}

\end{minipage}\quad
\begin{minipage}[t]{0.48\textwidth}         
\centering
\begin{tikzpicture}[scale=0.7]
	\def\xMin{-1}%
    \def\xMax{7}%
	\def\yMin{-1}%
    \def\yMax{5}%
    
	\draw [color=gray!50] [step=1] (\xMin+0.2,\yMin+0.2) grid (\xMax-0.2,\yMax-0.2);		
		
	\draw[->] (0,0,0) -- (\xMax-0.5,0,0) node[anchor=west]{$x$};
	\draw[->] (0,0,0) -- (0,\yMax-0.5,0) node[anchor=west]{$y$};
	
	\draw[very thick, pattern=north east lines, pattern color=gray!50] (2,0) -- (5,0) -- (5,1) -- (2,1) -- cycle;
	\node at (3.5,0.5) {$P_1$};
	\draw[very thick, fill=gray!50] (0,2) -- (1,2) -- (0,3) -- cycle;
	\node at (0.3,2.3) {$P_2$};	
	
	\draw[very thick] (2,2) -- (6,2) -- (6,3) -- (5,4) -- (2,4) -- cycle;
	\node at (4,3.3) {$P_1 + P_2$};
    \draw[pattern=north east lines, pattern color=gray!50] (2,2) -- (5,2) -- (5,3) -- (2,3) -- cycle;
    \draw[fill=gray!50] (5,3) -- (6,3) -- (5,4) -- cycle;
    
\end{tikzpicture}

\caption{Illustration of the Minkowski sum of two polytopes, relating to Example~\ref{ex:polytope:operations}.
}
\label{fig:example:Msum}
\end{minipage}
\end{figure}

\begin{thm}[{\citep[Thm.~3.35]{hertrich:phd} for $R=\R$ and \cite[Thm.~8]{hhl} for $R=\Z$}] \label{thm:polyh:criterion}
	Let $R$ be $\R$ or $\Z$. Then 
	\[
		\ReLU_n^{R,0}(k) = \{ h_A - h_B \colon A,B \in \SU^k(\cP_0(R^n)) \}. 
	\]
\end{thm}

\begin{cor} \label{cor:mink:difference} 
  Let $k \in \Z_+$ and $R$ be $\R$ or $\Z$.
  Let~$P \in \cP(R^n)$.
  Then, the support function $h_P$ of~$P$ satisfies~$h_P \in \ReLU_n^R(k)$ if and only if $P + A = B$ for some $A,B \in \SU^k(\cP_0(R^n))$. 
\end{cor} 
\begin{proof}
  By Theorem~\ref{thm:polyh:criterion}, we have that $h_P \in \ReLU_n^R(k)$ if and only if $h_P = h_B - h_A$ for some ${A,B \in \SU^k(\cP_0(R^n))}$.
  The equation $h_P = h_B - h_A$ can be rewritten as $h_B = h_P + h_A = h_{P+A}$, as support functions are Minkowski additive.
  Furthermore, every polytope is uniquely determined by its support function, see~\citep{HugWeil2020}, so $h_{P+A} = h_B$ is equivalent to $P + A = B$.
\end{proof}

The characterization of~$\ReLU_n^{R,0}(k)$ via~$\SU^k(\cP_0(R^n))$ as well as the geometric concepts of volumes will allow us to prove Theorem~\ref{thm:main1}.
The core step of our proof is to show that~$\maxpoolshort$, which is the support function of~$\Delta_n$, is not contained in~$\ReLU_n^{\Z,0}(k)$ for small~$k$.
As we will see later, it turns out to be useful to not work exclusively with full-dimensional polytopes in~$\SU^k(\cP_0(\Z^n))$, but with some of their lower-dimensional faces.
The next lemma provides the formal mechanism that we use, namely if~$P \in \SU^k(\cP_0(\Z^n))$ and~$F$ is a face of~$P$, then~$h_F \notin \ReLU^\Z_n(k)$ implies also~$h_P \notin \ReLU^\Z_n(k)$.
We defer the lemma's proof to Appendix~\ref{app:lem}.

\begin{restatable}{lemma}{missinglemma}
  \label{lem:face:su} 
	Let $k \in \Z_+$. Then,
	for all $P \in \SU^k(\cP_0(\Z^n))$ and $u \in \R^n$, the face of $P$ in direction~$u$, given by 
	\[
		P^u \coloneqq \{ x \in P \colon u^\top x = h_P(u) \}, 
	\]
	belongs to $\SU^k(\cP_0(\Z^n))$. In other words, $\SU^k(\cP_0(\Z^n))$ is closed under taking non-empty faces.
\end{restatable}

\section{Results and Proofs}
\label{sec:results}

The goal of this section is to prove Theorems~\ref{thm:main1} and~\ref{thm:main2} for the ring~$R$ of~$N$-ary fractions.
To this end, we will rescale~$\maxpoolshort$ by a suitable scalar~$\lambda \in \N$ such that the containment~$\maxpoolshort \in \ReLU_n^R(k)$ is equivalent to~$\lambda\maxpoolshort \in \ReLU_n^\Z(k)$.
To show that~$\lambda\maxpoolshort \notin \ReLU_n^\Z(k)$ if~$k$ is too small, we use a volume-based argument.
More precisely, we show that, for lattice polytopes~$P \subseteq \R^n$ whose support functions~$h_P$ are contained in~$\ReLU^\Z_n(k)$ and suitably defined dimensions~$d$ and prime numbers~$p$, their volumes~$\Vol_d(P)$ are divisible by~$p$.
In contrast, $\Vol_d(\lambda \Delta_n)$ is not divisible by~$p$, and thus, $\lambda \maxpoolshort \notin \ReLU^\Z_n(k)$.
This strategy is inspired by the proof of \citet{hhl} for~$\maxpoolshort \notin \ReLU_n^\Z(k)$, where related results are shown for the special case~$p=2$.
Our results, however, are more general and do not follow directly from their results.

To pursue this strategy, Sections~\ref{sec:aux1} and~\ref{sec:aux3} derive novel insights into volumes~$\Vol_d(P)$ of lattice polytopes~$P$ whose support functions~$h_p$ are contained in~$\ReLU_n^\Z(k)$.
These insights are then used in Section~\ref{sec:resultsFractions} to prove Theorems~\ref{thm:main1} and~\ref{thm:main2}.

\subsection{Divisibility of Normalized Volumes by a Prime}
\label{sec:aux1}

To understand the divisibility of~$\Vol_d$ by a prime number mentioned above, we investigate cases in which~${\Vol_d\colon \cP_d(\Z^n) \to \Z}$ modulo a prime is Minkowski additive.
To make this precise, we introduce some notation.

For $a,b \in \Z$ and $m \in \N$ we write $a \equiv_m b$ if $a-b$ is divisible by $m$. This is called the congruence of $a$ and $b$ modulo $m$. The coset $[z]_m$ of $z \in \Z$ modulo $m$ is the set of all integers congruent to~$z$ modulo $m$, and we denote the set of all such cosets by~$\Z_m$. The addition of cosets is defined by~$[a]_m + [b]_m \coloneqq [a+b]_m$ for $a,b \in \Z$. Endowing~$\Z_m$ with the addition operation $+$ yields a group of order $m$.

The following is an easy-to-prove but crucial observation. It states that when we consider lattice polytopes in a $d$-dimenensional subspace $Y \subseteq \R^n$ spanned by $d$ lattice points, the volume $\Vol_d$, taken modulo a prime number $p$, is an additive functional when $d$ is a power of $p$. 

\begin{prop} \label{prop:modular:volume}
  \primepowsentence
  Let~$P_1,\ldots,P_m\in \cP_d(\Z^n)$ be such that~$\sum_{i = 1}^m P_i \in \cP_d(\Z^n)$.
  Then,
\[
	\Vol_d \biggl( \sum_{i=1}^m P_i \biggr) \equiv_p \sum_{i=1}^m \Vol_d(P_i).
\] 
\end{prop}

\begin{proof}
	Note that by the assumption $\sum_{i=1}^m P_i \in \cP_d(\Z^d)$ all of the $P_i$'s lie, up to appropriate translation, in a $d$-dimensional vector subspace $Y$ of $\R^d$, which is spanned by $d$ lattice points. There is no loss of generality in assuming that $P_i \subseteq Y$ and, in view of  Remark~\ref{rem:d-dim:vol:theory}, we can use the mixed volume functional on $d$-tuples of polytopes from $\cP(Y)$, which will give an integer value for polytopes in $\cP(Y \cap \Z^n)$.
  By an inductive argument, it is sufficient to consider the case~$m=2$.
 It is well known that if $d$ is a power of $p$, the binomial coefficients $\binom{d}{1},\ldots,\binom{d}{d-1}$ in \eqref{MV:binomial} are divisible by~$p$, see, e.g., \citet[Cor.~2.9]{Mihet2010}.
Thus, \eqref{MV:binomial} implies $\Vol_d(P_1+P_2) \equiv_p \Vol_d(P_1) + \Vol_d(P_2)$ for~$P_1, P_2 \in \cP(Y \cap \Z^n)$.
\end{proof}

\begin{example}\label{ex:polytope:operations}
Consider the polytope $P_1 + P_2 \in \cP_2(\Z^2)$ for the rectangle $P_1=[2,5] \times [0,1] \in \cP_2(\Z^2)$ and the shifted standard simplex $P_2 = \Delta_2 + \{(0,2)\T\} \in \cP_2(\Z^2)$ as depicted in Figure~\ref{fig:example:Msum}.
In the picture, $P_1 + P_2$ is decomposed into regions in such a way that the volume of the mixed area $\V(P_1,P_2)$ can be read off.
In view of the equality $\Vol_2(P_1+ P_2) = \V(P_1 + P_2, P_1 + P_2) = \V(P_1,P_1)  + 2 \V(P_1,P_2) + \V(P_2,P_2) = \Vol_2(P_1) + 2 \V(P_1, P_2) + \Vol_2(P_2)$, see \eqref{MV:binomial}, the total volume of the unshaded part of $P_1 + P_2$ must be exactly $2\V(P_1,P_2)$.
For $p=2$ we have $\Vol_2(P_1 + P_2) = 15 ~{\equiv}_{2}~ 6 + 1 = \Vol_2(P_1) + \Vol_2(P_2)$, i.e., the parity of $\Vol_2(P_1 + P_2)$ is indeed that of $\Vol_2(P_1) + \Vol_2(P_2)$.
In contrast, divisibility by $p=3$ does not match, as $15 ~{\not\equiv}_{3}~ 7$. However, this does not contradict Proposition~\ref{prop:modular:volume}, as $d=2$ is not a power of $p=3$.
\end{example}

To derive divisibility properties of~$\Vol_d(P)$ for lattice polytopes~$P$ with~$h_P \in \ReLU_n^\Z(k)$, we make use of the characterization of~$\ReLU_n^\Z(k)$ via the~$\SU$-operator.
Recall that one of the two defining operations of~$\SU$ is~$\conv(A \cup B)$ for suitable polytopes~$A$ and~$B$.
A crucial observation is that for certain dimensions~$d$, the divisibility of~$\Vol_d(\conv(A \cup B))$ by a prime number is inherited from particular lower-dimensional faces of~$A$ and~$B$.

\begin{restatable}{proposition}{missingprop} \label{prop:union}
  \primepowsentence Moreover, let~$P = \conv(A \cup B) \in \cPd(\Z^n)$ for $A, B \in \cPd(\Z^n)$.
  If $\Vol_{p^{t-1}}(F) \equiv_p 0$ for all $p^{t-1}$-dimensional faces $F$ of $A$ and $B$, then $\Vol_{p^t}(P) \equiv_p 0$.
\end{restatable}

Note that this result also holds trivially if no face of dimension~$p^{t-1}$ exists.
We defer the proof of this result to Appendix~\ref{app:prop}.

\subsection{Modular Obstruction on Volume for Realizability With $k$ Hidden Layers} 
\label{sec:aux3}

Equipped with the previously derived results, we have all ingredients together to prove the aforementioned results on the divisibility of~$\Vol_d(P)$ for lattice polytopes~$P$ with~$h_P \in \ReLU_n^\Z(k)$.

\begin{thm} \label{thm:p:invariant:relu}
  \primepowsentence
  Let $k \in [t]$ and $P \in \SU^k(\cP_0(\Z^n))$. Then $\Vol_{p^k}(F) \equiv_p 0$ for all $p^k$-dimensional faces $F $of $P$.
\end{thm}

\begin{proof} 
  We argue by induction on $k$. If~$k=1$, then~$\SU^1(\cP_0(\Z^n))$ consists of lattice zonotopes.
  These are polytopes of the form $P  = S_1+ \cdots +S_m$, where $S_1,\ldots,S_m$ are line segments joining a pair of lattice points. One has 
	$\Vol_d(P) \equiv_p \Vol_d( \sum_{i=1}^m S_i) \equiv_p \sum_{i=1}^m\Vol_d(S_i)$, by  Proposition~\ref{prop:modular:volume}, arriving at $\Vol_d(P) \equiv_p 0$, since $\Vol_d(S_i) = 0$ for all $i$ as~$d > 1 \geq \dim(S_i)$. 

        In the inductive step, assume~$k \ge 2$ and that the assertion has been verified for $\SU^{k-1}(\cP_0(\Z^n))$.
        Recall that every $P \in \SU^k(\cP_0(\Z^n))$ can be written as~$P = \sum_{i=1}^m \conv(A_i \cup B_i)$ for some polytopes $A_i, B_i \in \SU^{k-1}(\cP_0(\Z^n))$. 
	By the induction hypothesis, the $p^{k-1}$-dimensional normalized volumes of the  $p^{k-1}$-dimensional faces of $A_i$ and $B_i$ are divisible by $p$. Consequently, by Proposition~\ref{prop:union}, the $p^k$-dimensional normalized volumes of the $p^k$-dimensional faces of $\conv(A _i \cup B_i)$ are divisible by $p$.
        Since $\SU^k(\cP_0(\Z^n))$ is closed under taking faces (see Lemma~\ref{lem:face:su}), Proposition~\ref{prop:modular:volume} applied to the $p^k$-dimensional faces of $P$ implies that the $p^k$-dimensional normalized volume of the $p^k$-dimensional faces of $P$ is divisible by $p$. 
\end{proof} 

\begin{thm} \label{vol:relu:thm}
  \primepowsentence
  Let $P$ be a lattice polytope in $\cP_d(\R^n)$.
  If~$h_P \in \ReLU^\Z_n(k)$, $k \in [t]$, then $\Vol_d(P)$ is divisible by $p$. 
\end{thm} 
\begin{proof} 
By Corollary~\ref{cor:mink:difference}, we have $P + A = B$ for some $A, B \in \SU^k(\cP_0(\Z^n))$. Then, by Proposition~\ref{prop:modular:volume}, one has  $\Vol_d(P+A) \equiv_p \Vol_d(P) + \Vol_d(A) \equiv_p \Vol_d(B)$, which means that $\Vol_d(P) \equiv_p \Vol_d(A) - \Vol_d(B)$. By Theorem~\ref{thm:p:invariant:relu}, both $\Vol_d(A)$ and $\Vol_d(B)$ are divisible by $p$. This shows that $\Vol_d(P)$ is divisible by $p$. 
\end{proof} 

\subsection{Proofs of Main Results}
\label{sec:resultsFractions}

We now turn to the proofs of Theorems~\ref{thm:main1} and~\ref{thm:main2}.
Let $N \in \N$ and recall that a rational number is an $N$-ary fraction if it is of the form $\frac{z}{N^t}$ with $z \in \Z$ and $t \in \Z_+$. For $N=2$ and $N=10$, one has binary and decimal fractions, used in practice in floating point calculations. Clearly, every binary fraction is also a decimal fraction, because $\frac{z}{2^t} = \frac{5^t z }{10^t}$. 

In order to relate ReLU networks with fractional weights to ReLU networks with integer weights, we can simply clear denominators, as made precise in the following lemma.

\begin{lem} \label{lem:clear:denominator}
Let $f \colon \R^n \to \R$ be exactly representable by a ReLU network with~$k$ hidden layers and with rational weights all having~$M$ as common denominator.
Then $M^{k+1} f \in \ReLU_n^{\Z}(k)$.
\end{lem}

\begin{proof}
  We proceed by induction on $k$. For the base case~$k=0$, $f$ is an affine function~$f(x_1,\ldots,x_n) = b + a_1 x_1 + \cdots + a_n x_n$ with $b \in \R$ and $M a_1,\ldots,M a_n \in \Z$, from which the claim is immediately evident.
Now let $k \geq 1$ and consider a $k$-layer network with rational weights with common denominator $M$ representing~$f$.
Then $f$ is of the form $f(x) = u_0 + u_1 \max \{0, g_1(x) \} + \cdots + u_m \{0, g_m(x) \}$ with $m \in \N$, where all $g_1,\ldots,g_m$ are functions representable with $k-1$ hidden layers and all the weights, i.e., $u_1,\ldots,u_m$ and the ones used in expressions for $g_1,\ldots,g_m$, are rational numbers with $M$ as a common denominator. Multiplying with~$M^{k+1}$ we obtain
\[
	M^{k+1} f(x) = M^{k+1} u_0 + M u_1 \cdot \max \{0 , M^{k} g_1(x) \} + \ldots + M u_m \cdot \max \{0 , M^{k} g_m(x) \},
\] 
where the weights $M u_1,\ldots, M u_m$ are integer. By the induction hypothesis, for every~$i \in [m]$, we have $M^{k} g_i \in \ReLU_n^{\Z}(k-1)$, and hence $M^{k+1} f \in \ReLU_n^\Z(k)$.
\end{proof}

We are now ready to prove our main results.

\begin{proof}[Proof of Theorem~\ref{thm:main1}]
  Let~$k = \ceil{\log_p(n+1)} - 1$, i.e., $k$ is the unique non-negative integer satisfying~$p^{k} < n+1 \leq p^{k+1}$.
  If~$\maxpoolshort$ was representable by a ReLU network with $k$ hidden layers and~$N$-ary fractions as weights, $\maxpool[p^k] = \maxpoolshort(x_1,\ldots,x_{p^k},0,\ldots0)$ would also be representable in this way. 
  Thus, it suffices to consider the case $n=p^k$.

  Recall that~$\maxpoolshort$ is the support function $h_{\Delta_n}$ of the standard simplex.
  Suppose, for the sake of contradiction, that $\maxpoolshort$ can be represented by a ReLU network with $k$ hidden layers and weights being~$N$-ary fractions. Let $t \in \N$ be large enough such that all weights are representable as $\frac{z}{N^t}$ for some $z \in \Z$. We use Lemma~\ref{lem:clear:denominator} with $M=N^t$ to clear denominators. That is, $N^{t(k+1)} \maxpoolshort$ is representable by an integer-weight ReLU network with $k$ hidden layers. Since $\maxpoolshort$ is homogeneous, we can assume that the network is homogeneous, too \citep[Proposition~2.3]{hbds}.
  Observe that~$N^{t(k+1)} \maxpoolshort$ is the support function of $N^{t(k+1)} \Delta_n$, whose normalized volume satisfies $\Vol_n(N^{t(k+1)} \Delta_n) \equiv_p N^{n t(k+1)} \Vol_n(\Delta_n) = N^{n t(k+1)} \cdot 1 \not\equiv_p 0$. Hence, $N^{t(k+1)} \Delta_n$ is a polytope in~$\R^{p^k}$ whose normalized volume is not divisible by $p$, but whose support function is represented by an integer-weight ReLU network with $k$ hidden layers. This contradicts Theorem~\ref{vol:relu:thm}. Hence, $\maxpoolshort$ is not representable by a ReLU network with $k$ hidden layers and weights being~$N$-ary fractions. 
\end{proof}

If $N=10$, we can use $p=3$ in Theorem~\ref{thm:main1}, so
Corollary~\ref{cor:log3} is an immediate consequence.
The bound $\ceil{\log_3 (n+1)}$ in Corollary~\ref{cor:log3} is optimal up to a constant factor, as $\maxpoolshort$ is representable through a network with integral weights and $\lceil \log_2 (n+1) \rceil$ hidden layers~\citep{abmm}. A major open  question raised by \citet{hbds} is whether this kind of result can be extended to networks whose weights belong to a larger domain like the rational numbers or, ideally, the real numbers. 

We can also show that the expressive power of ReLU networks with weights being decimal fractions grows gradually when the depth $k$ is increasing in the range from $1$ to $\bigO(\log n)$. 

\begin{cor}
	For each $n \in \N$ and each integer $k \in \{1,\ldots, \lceil  \log_3 n \rceil \}$, within $n$-variate functions that are described by ReLU networks with weights being decimal fractions, there are functions representable using $2k$ but not using $k$ hidden layers. 
\end{cor} 
\begin{proof}
  Function~$F_{3^k}$ is not representable through $k$ hidden layers and weights being decimal fractions. Since $3^k \le 2^{2k}$, $F_{3^k}$ is representable with $2k$ hidden layers (and integer weights). 
\end{proof} 

By making use of Theorem~\ref{thm:main1}, we now present the proof of Theorem~\ref{thm:main2}.

\begin{proof}[Proof of Theorem~\ref{thm:main2}]

To make use of Theorem~\ref{thm:main1}, we need to find a prime number~$p$ that does not divide~$N$.
Let~$p_i$ denote the~$i$-th prime number, i.e., $p_1 = 2, p_2 = 3, p_3 = 5$ etc.
Moreover, assume that the prime number decomposition of~$N$ consists of~$t$ distinct primes, i.e.,
$N = p_{i_1}^{m_1} \cdots p_{i_t}^{m_t}$ where $m_1,\ldots,m_t \in \N$ and $i_1 < \cdots < i_t$.
Choose the minimal prime $p$ that is not contained in $\{p_{i_1},\ldots,p_{i_t}\}$, that is, the minimal prime not dividing $N$.
Since $\{p_1,\ldots,p_{t+1}\}$ has a prime not contained in $\{p_{i_1},\ldots,p_{i_t}\}$, we see that $p \le p_{t+1}$.

To get a more concrete upper bound on~$p$, we make use of the prime number theorem \citep[Ch.~XXII]{HardyWright2008}, which is a fundamental result in number theory.
The theorem states that~$\lim_{i \to \infty} \frac{p_i}{i \ln i} = 1$.
Consequently, $p \leq p_{t+1} \le 2 t \ln t$ when $t \ge T$, where $T \in \N$ is large enough. We first stick to the case $t \ge T$ and then handle the border case $t < T$. 

For $\ln N$ we have 
\begin{align*} 
	\ln N  & =  \sum_{j=1}^{t} m_j\ln p_{i_j} 
	\ge \sum_{j=1}^{t} \ln p_{i_j}
		  \ge \sum_{j=1}^t \ln (j+1)
		\ge  \int_1^{t+1} \ln x \operatorname{d} x  = (t+1) \ln (t+1) - t
\end{align*}
for all $t \ge T$.  Thus, $\ln N \ge \nicefrac{1}{2} t \ln t$.
This implies $\ln \ln N \ge \ln t + \ln \ln t - \ln 2$. Compare this to $\ln p \le \ln 2 + \ln t + \ln \ln t$. So, we see that   $\ln \ln N \ge C \ln p$ with an absolute constant $C>0$. Hence, we can invoke Theorem~\ref{thm:main1} for $p$, getting that the number of layers needed to represent $\maxpoolshort$ with integer weights is at least $\log_p n$, where $\log_p n \ge \nicefrac{\ln n}{\ln p} \ge C \cdot \nicefrac{\ln n}{\ln \ln N}$. In the case $t < T$, we observe that $p \leq p_T$ and obtain the lower bound $\log_p n = \nicefrac{\ln n}{\ln p} \geq \nicefrac{\ln n}{\ln p_T}$. Since $T$ is fixed, the factor $\ln p_T$ in the denominator is an absolute constant. 
\end{proof} 

\section{Conclusions}
\label{sec:conclusions}

In summary, we proved that a lower bound for the number of hidden layers needed to exactly represent the function~$\maxpool$ with a ReLU network with weights being~$N$-ary fractions is $\ceil{\log_p (n + 1)}$, where $p$ is a prime number that does not divide $N$. For $p=3$, this covers the cases of binary fractions as well as decimal fractions, two of the most common practical settings.
Moreover, it shows that the expressive power of ReLU networks grows for every $N$ up to $\bigO(\log n)$. In the case of rational weights that are $N$-ary fractions for any fixed $N$, even allowing arbitrarily large denominators and arbitrary width does not facilitate exact representations of low constant depth.

Theorem~\ref{thm:main2} can be viewed as a partial confirmation of Conjecture~\ref{conj:central} for rational weights, as the term $\ln \ln N$ is growing extremely slowly in $N$. If one could replace $\ln \ln N$ by a constant, the conjecture would be confirmed for rational weights, up to a constant multiple.
As already highlighted in \citet{hhl}, confirmation of the conjecture would theoretically explain the significance of max-pooling in the context of ReLU networks: It seems that the expressive power of ReLU is not enough to model the maximum of a large number of input variables unless network architectures of high-enough depth are used. So, enhancing ReLU networks with max-pooling layers could be a way to reach higher expressive power with shallow networks. 
  
Methodologically, algebraic invariants -- such as the $d$-dimensional volume $\Vol_d$ modulo a prime number $p$ when $d$ is a power of $p$ -- play a key role in showing lower bounds for the depth of neural networks. Our proof approach provides an algebraic template for a general ``separation strategy'' to tackle problems on separation by depth. In the ambient Abelian group $(G,+)$ of all possible functions that can be represented within a given model, one has a nested sequence of subgroups $G_0 \subseteq G_1 \subseteq G_2 \subseteq \cdots $, with $G_k$ consisting of functions representable by $k$ layers. To demonstrate that an inclusion $G_k \subseteq G_{k+1}$ is strict, one could look for an invariant that can distinguish $G_k$ from $G_{k+1}$ -- this is a group homomorphism $\phi$ on $G$ that is zero on $G_k$ but not zero on some $f \in G_{k+1}$.  Most likely, the invariant needs  to be  ``global'' in the sense that, if $\phi(f)$ is computed from the NN representation of $f$, then it would accumulate the contribution of all the nodes of the NN in one single value and would not keep track of the number of the nodes and, by this, disregard the widths of the layers. In the concrete case we handled in this contribution, the group $G$ we choose is $\ReLU^{\Z,0}$, whereas the invariant that we employ has values in $\Z_p$ and is based on the computation of the volume of lattice polytopes. In the original setting of Conjecture~\ref{conj:central}, one has to deal with the nested chain of subspaces $G_k = \ReLU^{\R,0}(k)$ of the the infinite-dimensional vector space $G = \ReLU^{\R,0}$, which makes it natural to choose as an invariant a linear functional $\phi\colon G \to \R$. To make further progress, it is therefore worthwhile continuing to investigate the connection between ReLU networks and discrete polyhedral geometry, algebra, and number theory in order to settle Conjecture~\ref{conj:central} in general. 

Finally, we raise a question on the role of the field of real numbers in  Conjecture~\ref{conj:central}. Does the choice of a subfield of $\R$ matter in terms of the expressiveness? More formally, we phrase 
\begin{question}
	Let $S$ be a subfield of $\R$ and $k \in \N$ and let $f\colon \R^n \to \R$ be a function expressible via a ReLU network with weights in $S$. If $f$ is expressible by a ReLU network with $k$ hidden layers and weights in $\R$, is it also expressible by a ReLU network with $k$ hidden layers and weights in $S$? What is the answer for $S = \Q$?
\end{question} 
If, for $S = \Q$, the answer to the above question is positive, then the version of Conjecture~\ref{conj:central} with rational weights is equivalent to the original conjecture with real weights, which might be a helpful insight towards solving Conjecture~\ref{conj:central}. If the answer is negative, then the conjecture would have a subtle dependence on the underlying field of weights. 

\newpage


\subsubsection*{Acknowledgments}
We would like to thank Sergey Grosman from Siemens (Konstanz, Germany) for his feedback on the practical role of quantization,
and five anonymous reviewers for their helpful comments.
The research of GA was partially supported by the DFG Project 539867386
``Extremal bodies with respect to lattice functionals'', which is carried
out within the DFG Priority Program 2458 ``Combinatorial Synergies''.


\bibliography{rational_relu}
\bibliographystyle{iclr2025_conference}

\appendix
\section{Appendix}

\subsection{Deferred Proofs}

In this appendix, we provide the proofs missing in the main part
of the article.
For convenience of reading, we restate the corresponding statements.

\subsubsection{Proof of Lemma~\ref{lem:face:su}}
\label{app:lem}

This appendix provides the missing proof of the following lemma.

\missinglemma*

\begin{proof}
  Throughout the proof, let~$\cX = \cP_0(\Z^n)$.
  The proof is by induction on $k$.
  For $k=0$, we have $\SU^0(\cX) = \cX$.
  Since every polytope in~$\cP_0(\Z^n)$ consists of a single point~$s$, every non-empty face of such a polytope also just consists of~$s$, and is therefore contained in~$\cP_0(\Z^n)$.
  Thus, the claim holds.

Now let $k\geq 1$ and assume the assertion holds for $k-1$.
Furthermore, let $u \in \R^n$ and~${P \in \SU^k(\cX)}$ with $P = \sum_{i=1}^m \conv(A_i \cup B_i)$ for some $m \in \N, \ A_i, B_i \in \SU^{k-1}(\cX), i \in [m]$.
By definition and Minkowski additivity of the support function, we have $P^u = (\sum_{i=1}^m \conv(A_i \cup B_i))^u = \sum_{i=1}^m (\conv(A_i \cup B_i))^u$. Moreover, for each $i \in [m]$, $\conv(A_i \cup B_i)^u$ is equal to $A_i^u$, $B_i^u$, or $\conv(A_i^u \cup B_i^u)$ depending on whether $h_{A_i}(u) > h_{B_i}(u)$, $h_{A_i}(u) < h_{B_i}(u)$, or $h_{A_i}(u) = h_{B_i}(u)$, respectively. In any case, we obtain a representation of $P^u$ that shows its membership in $\SU^k(\cX)$, since $A_i, B_i \in \SU^{k-1}(\cX)$ for all $i \in [m]$ by the induction hypothesis.
\end{proof}

\subsubsection{Proof of Proposition~\ref{prop:union}}
\label{app:prop}

The goal of this section is to prove the following statement.

\missingprop*

To prove this result, we need two auxiliary results that we provide next.

\begin{prop} \label{prop:vol:faces}
	Let $m, s, d \in \N$ and $s < d \leq n$. 
	If $P \in \cPd(\Z^n)$ such that $\Vol_s(F) \equiv_m 0$ for all $s$-dimensional faces $F$ of $P$, then $\Vol_d(P) \equiv_m 0$. 
\end{prop}
\begin{proof}
	Note that we can restrict our attention to the case $d=s+1$:
	Once the case $d=s+1$ is settled, it follows that the divisibility of $\Vol_i(F)$ by $m$ for $i$-dimensional faces $F$ of $P$ implies divisibility of $\Vol_{i+1}(G)$ by $m$ for all $(i+1)$-dimensional faces $G$ of $P$. Hence, iterating from $i=s$ to $i=d-1$, we obtain the desired assertion. So, assume $d = s+1$.
	
	Let $P$ be a $d$-dimensional lattice polytope with facets having a normalized $(d-1)$-dimensional volume divisible by $m$. We pick a vertex~$a$ of $P$ and subdivide $P$ into the union of the non-overlapping pyramids of the form $\conv(\{a \} \cup F)$, where $F$ is a facet of $P$. By Proposition~\ref{prop:vol:join}, the normalized $d$-dimensional volume of $\conv( \{a\} \cup F)$ is divisible by $\Vol_{d-1}(F)$.
	 Since by assumption $\Vol_{d-1}(F)$ is divisible by $m$, we conclude that also $\Vol_d(P)$ is divisible by $m$, because~$\Vol_d$ is additive as it is based on a Lebesgue measure. 
\end{proof} 

The second result analyzes the structure of~$\conv(A \cup B)$ in terms of a
particular subdivision.
Given a polytope~$P \in \cP(\R^n)$ of dimension~$d$, a \emph{subdivision}
of~$P$ is a finite collection~$\mathcal{C} \subseteq \cP(\R^n)$ such that (i) $P =
\bigcup_{C \in \mathcal{C}} C$; (ii) for each~$C \in \mathcal{C}$,
the polytope~$C$ has dimension~$d$;
(iii) for any two distinct~$C,C' \in \mathcal{C}$, the polytope~$C \cap C'$
is a proper face of both~$C$ and~$C'$.
The elements~$C \in \mathcal{C}$ are called the \emph{cells} of
subdivision~$\mathcal{C}$, cf.~\citep{LeeSantos2017}.

\begin{prop}[{\citealt[Prop.~10]{hhl}}] \label{prop:subdivision}
For two polytopes $A, B \in \cP(\R^n)$, there exists a subdivision of
$\conv(A \cup B)$ such that each full-dimensional cell is of the form $\conv(F \cup G)$, where~$F$ and~$G$ are faces of~$A$ and~$B$, respectively, such that~$\dim(F) + \dim(G) + 1 =d$.
\end{prop}

The term \enquote{full-dimensional} in Proposition~\ref{prop:subdivision} as well as in the original formulation of \citet[Prop.~10]{hhl} refers to faces that have the same dimension as $\conv(A \cup B)$, while its authors make no assumption on whether that dimension is equal to $n$ (but \citet{hhl} note in their proof that such an assumption would be without loss of generality).

We are now able to prove Proposition~\ref{prop:union}.

\begin{proof}[Proof of Proposition~\ref{prop:union}]
  Let $P = \conv(A \cup B)$.
  We apply Proposition~\ref{prop:subdivision} for obtaining a subdivision of $P$ into $d$-dimensional polytopes $P_1=\conv(F_1 \cup G_1),\ldots,P_m=\conv(F_m \cup G_m)$, where for each $s\in [m]$, $F_s$ and~$G_s$ are faces of~$A$ and~$B$, respectively, and $\dim(F_s) + \dim(G_s) + 1 =d$.
  That is, $P$ is the union of polytopes whose relative interiors are disjoint.
  Consequently, $\Vol_d(P) = \Vol_d(P_1) + \cdots + \Vol_d(P_m)$.
  It therefore suffices to show that $\Vol_d(P_s) \equiv_p 0$ for every such polytope $P_s$ with $s \in [m]$.  

  For given $s \in [m]$ and faces~$F_s$ and~$G_s$ of~$A$ and~$B$, respectively, denote their dimensions as $i$ resp.~$j$. Since $i +j = d-1 = p^t - 1$, the dimension of $F_s$ or $G_s$ is at least $p^{t-1}$ (if this was not the case, we would have $i +j \le 2 (p^{t-1} -1) < p^t-1$, which is a contradiction). By symmetry reasons, we assume without loss of generality that $i \ge p^{t-1}$. Then, by Proposition~\ref{prop:vol:faces}, $\Vol_i(F_s)$ is divisible by $p$.  Consequently, by Proposition~\ref{prop:vol:join}, the normalized volume of $\conv(F_s \cup G_s)$ is also divisible by~$p$.
\end{proof}

\subsection{Proof of Binomial Formula for Mixed Volumes}
\label{sec:ProofBinomial}

The symmetry and multilinearity of the mixed-volume functional makes computations with it similar in nature to calculations with an $n$-term product. Say, the identity $(x+y)^2 = x^2 + 2 x y + y^2$ over reals corresponds to the identity $\Vol_2(A+B) = \V(A+B,A+B) = \V(A,A) + 2 \V(A,B) + \V(B,B) = \Vol_2(A) + 2 \V(A,B) + \Vol_2(B)$ for planar polytopes $A,B$ and the way of deriving the latter identity is completely analogous to deriving the identity for $(x+y)^2$ by expanding brackets. Very much in the same way, the binomial identity $(x+y)^n = \sum_{i=0}^n \binom{n}{i} x^i y^{n-i}$ corresponds to the identity \eqref{MV:binomial}. Here is a formal proof:

We use the notation $P_0 = B$ and $P_1 = A$. Then 
	\[
		\Vol_n(P_0 + P_1) = \V( P_0 + P_1, \ldots, P_0 + P_1)
	\]
	by Property (c) in Theorem~\ref{thm:mixed:vol}. Using Property (b) in Theorem~\ref{thm:mixed:vol} for each of the $n$ inputs of the mixed-volume functional, we obtain
	\[
		\Vol_n(P_0+ P_1) = \sum_{i_1 \in \{0,1\}} \cdots \sum_{i_n \in \{0,1\}} \V(P_{i_1},\cdots,P_{i_n}),
	\]
	where the right-had side is a sum with $2^n$ terms. However, many of the terms are actually repeated, because $\V(P_{i_1},\ldots,P_{i_n})$ does not depend on the order of the polytopes in the input: this mixed volume contains $i_1+ \cdots + i_n$ copies of $P_1$ and $n - (i_1 + \cdots + i_n)$ copies of $P_0$. Hence, 
	\[
		\V(P_{i_1},\ldots, P_{i_n}) =  \V(\underbrace{P_0,\ldots, P_0}_{n - (i_1+ \cdots i_n)}, \underbrace{P_1,\ldots,P_1}_{i_1 + \cdots + i_n}).
	\]
	In order to convert our $2^n$-term sum into an $(n+1)$-term sum, for each choice of $i = {i_1 + \cdots + i_n} \in \{0,\ldots,n\}$, we can determine the number of choices of $i_1,\ldots,i_n \in \{0,1\}$ that satisfy $i = i_1 + \cdots + i_n$. This corresponds to choosing an $i$-element subset $\{t \in [n] \colon i_t = 1\}$ in the $n$-element set $\{1,\ldots,n\}$. That is, the number of such choices is the binomial coefficient $\binom{n}{i}$.  Hence, our representation with $2^n$ terms amounts to 
	\[
		\Vol_n(P_0 + P_1)  = \sum_{i=0}^n \binom{n}{i} \V(\underbrace{P_0,\ldots,P_0}_{n-i},\underbrace{P_1,\ldots,P_1}_i). 
	\]

\end{document}